\documentclass[twocolumn,letterpaper]{article}

\usepackage{times}
\usepackage{epsfig}
\usepackage{graphicx}
\usepackage{amsmath}
\usepackage{amssymb}
\usepackage{amsthm}
\usepackage{subfigure}

\usepackage{mathtools}

\newtheorem{thm}{Theorem}[section]

\newtheorem{assume}[thm]{Assumption}

\newcommand{\spaceo}{\hspace{2 mm}}
\newcommand{\setsep}{ \spaceo | \spaceo}
\newcommand{\half}{\frac{1}{2}}

\newcommand{\Abs}[1]{\left| #1 \right|}
\newcommand{\Set}[1]{\left\{ #1 \right\}}
\newcommand{\Brack}[1]{\left( #1 \right)}

\newcommand{\Expsubidx}[2]{ \mathbb{E}_{#1} #2}

\begin{document}

\title{Multi Instance Learning For Unbalanced Data}

\author{Mark Kozdoba \\ 
Technion,IIT \\
{\tt\small markk@technion.ac.il} \\
\and 
Edward Moroshko \\
Technion,IIT \\
{\tt\small edward.moroshko@gmail.com}
\and 
Lior Shani\\ 
Technion,IIT\\
{\tt\small shanlior@gmail.com}
\and
Takuya Takagi\\ 
Fujitsu Laboratories\\
{\tt\small takagi.takuya@jp.fujitsu.com}
\and 
Takashi Katoh\\
Fujitsu Laboratories\\
{\tt\small kato.takashi\_01@jp.fujitsu.com}
\and 
Shie Mannor\\
Technion,IIT\\
{\tt\small shie@ee.technion.ac.il}
\and 
Koby Crammer\\
Technion,IIT \\
{\tt\small koby@ee.technion.ac.il}
}

\date{}

\maketitle

\begin{abstract}
In the context of Multi Instance Learning, we analyze the Single Instance (SI) learning objective. We show that when the data is unbalanced and the family of classifiers is sufficiently rich, the SI method is a useful learning algorithm. 
In particular, we show that 
larger data imbalance, a quality that is typically perceived as negative, in fact implies a better resilience of the algorithm to the statistical dependencies of the objects in bags. In addition, 
our results shed new light on some known issues with the SI method in the setting of linear classifiers, and we show that these issues are significantly less likely to occur in the setting of neural networks. 
We demonstrate our results on a synthetic dataset, and on the COCO dataset for the problem of patch classification with weak image level labels derived from captions.
\end{abstract}

\section{Introduction}
Multi Instance Learning (MIL) is a generalization of supervised learning, 
where the data is given as \textit{bags}, and each bag is a set of 
objects. Each object can be either positive or negative, but we are not 
given this information. Instead, we are given the label for a bag as a 
whole, such that the bag is positive if and only if at least one of the 
elements in the bag is positive. The goal is to learn the 
\textit{instance  classifier} -- a classifier for objects, using only the 
bag labels. 

In recent years, there has been a constant stream of work on the MIL
problem. We refer to \cite{MILSurv1} and to \cite{MILSurv} for a survey 
of recent results and applications. 

One natural and important application of MIL is in the domain of images with weak 
labels. Here, one considers a large image, which may contain several 
objects, such as ``car'' or ``tree'', but the location of the objects in 
the image is not specified. In this case, one may divide the image into 
smaller overlapping patches, which together constitute a bag, such that 
some of the patches correspond to some of the labels. The labels 
themselves can be derived from some text related to the image, such as 
captions in the COCO dataset. This scheme was, for instance, an
important part of the automatic image description generation methods, 
such as \cite{karpathy15}, \cite{Fang_2015_CVPR}, but has numerous other 
applications.

One approach to the MIL problem is via the Single Instance (SI) method.
In this method, one simply \textit{unpacks} the bags, and considers a 
supervised learning problem where the data is the set of all objects from 
all the bags, and the label of each object is the label of the bag from 
which the object was extracted. In what follows we refer to this 
assignment of labels to objects as the SI assignment.

An oft-cited advantage of the SI method is its conceptual and 
computational simplicity. However, perhaps an even more important 
advantage is scalability. Indeed, non-SI MIL methods typically compute 
a certain score for each bag, which depends on individual scores of the 
objects in it. In oder to compute this score, one therefore may need to 
design iterative procedures if the bag is too large to fit in a single 
batch. In contrast, in SI the objects are no longer tied to a bag, 
and each bag can be divided into independent batches. Additional details 
are given in Section \ref{sec:conclusion}.

While the SI method was empirically investigated in the literature, 
there seems to be no complete picture with regards to when the method is 
effective and why, and at the same time significant efforts are made 
construct new and highly involved MIL methods.

In this paper, we show that in the case of unbalanced labels, and when 
the class of classifiers is rich enough, the SI method is effective. We 
now describe the results in more detail. 

Let $P$ and $N$ be the numbers of positive and negative bags in the dataset, respectively. Let $B$ be the ratio, such that $N = B \cdot P$. 
We call the dataset unbalanced if $B$ is large. For instance, 
in the COCO dataset, for a label ``car'', $B$ is about $30$.

An additional dataset characteristic, that affects the performance of all 
MIL algorithms, is the amount of intra-bag dependence. Roughly speaking, 
we say that the dataset has a low intra-bag dependence if the negative 
features in positive bags look like generic negative features. The full 
definition is given in Section \ref{sec:feature_dep_bags}, where we refer 
to this as the mixing assumption. It is known empirically, and for some 
methods theoretically, that under this assumption the MIL problem 
is relatively easy. Here we show that this is also the case for the SI 
method.

More importantly, however, we analyze the SI method in cases where the 
mixing assumption does \textit{not} hold. In these cases, we find that 
the larger the imbalance constant $B$ is, the more tolerant the SI 
method is to the lack of mixing. Since many natural datasets exhibit 
lack of mixing, but also data imbalance, it follows that the SI method 
is expected to perform well. The lack of mixing for images data in 
particular is discussed in Section \ref{sec:feature_dep_bags}.

Finally, as discussed in the Literature section in more detail, 
the evaluation of SI in the literature is done with \textit{linear} 
classifiers. On the other hand, our results, Theorems 
\ref{thm:main_thm} and \ref{thm:main_thm_gen} express the optimizer of the 
SI objective as a certain functional related to the optimal instance 
classifier. This functional, however, is rarely a linear classifier even 
if the optimal instance classifier is. This strongly suggest that to use 
the SI method, at least some non-linearity should be added, and that the 
SI method is particularly well suited to be used with neural networks.

In Section \ref{sec:experiments} we perform experiments on synthetic data, 
and on the COCO dataset with captions as weak labels. On the synthetic 
data, we demonstrate that the SI method is indeed tolerant to the lack of 
mixing, and that the use of a one hidden layer neural network  
significantly improves the results in comparison to a linear classifier, \textit{even when the ground truth data is linearly separable}. We also employ this example to 
show that one possible alternative to the SI method, 
based on noisy label methods (see the discussion in Sections 
\ref{sec:literature}, \ref{sec:feature_dep_bags}), is strongly sensitive 
to the lack of mixing.  In the COCO experiment, we reproduce the MIL 
setting of \cite{Fang_2015_CVPR}, with a 1000 tokens from captions as bag 
labels, and compare the SI objective with the 
soft-nor objective used in \cite{Fang_2015_CVPR}. Since in this setting 
one can not measure instance level performance, due to the lack of ground 
truth, we measure bag-level performance. Our results show that both 
objectives have a very similar performance, although the SI results are 
slightly lower. Possible reasons for this are discussed.

To summarize, the contributions of this paper are as follows: We provide a 
large-sample analysis 
of the SI method, and show (a) The optimal instance classifier can be 
obtained from the optimal SI assignment classifier simply by 
thresholding at an appropriate level. (b) The balance $B$ of the data 
plays an important role. The more unbalanced the data is (larger $B$) 
the more tolerant the algorithm becomes to data dependencies inside bags. 
To the best of our knowledge, these results are new and in particular the 
important role of the balance has not been previously noted.   
Next, we provide a link between the performance of the SI method and the 
richness of the classifier class and show that the SI method is 
particularly well suited to work with neural network classifiers.
Finally, we show that an SI method achieves performance comparable to the 
state-of-the art on image and captions data, and in addition demonstrate 
the tolerance of the SI to dependencies in bags, to support our 
theoretical results.

The rest of this paper is organized as follows: 
In Section \ref{sec:literature} we review the related literature. Section 
\ref{sec:si_analyis} contains the main results. In Section 
\ref{sec:experiments} the experiments on synthetic data, comparison to a 
noisy label classifier, and the experiment on COCO data are presented. 
We conclude the paper in Section \ref{sec:conclusion} with a discussion  
of possible future research directions.

\section{Literature}
\label{sec:literature}
As discussed in the Introduction, the field of MIL has generated a large 
amount of  interest and is still growing. General surveys can be found in 
\cite{MILSurv1} and in the very recent \cite{MILSurv}. Examples of some  
recent work related to, or using MIL methods may be found in 
\cite{Hoffman_2015_CVPR}, \cite{Wu_2015_CVPR},\cite{Li_2015_CVPR},
\cite{karpathy15},\cite{Fang_2015_CVPR} and \cite{attentionMIL}.

We now discuss specifically the SI related literature. SI methods were 
empirically evaluated and compared to other methods in \cite{Ray}, 
and more recently in \cite{Alpaydin}. In \cite{Ray}, it was found 
that in many cases, the SI methods yield the most competitive results. 
It is of interest to note that the evaluation in \cite{Ray} is done with 
linear classifiers. As discussed in the Introduction and shown in Section 
\ref{sec:niosy_label_estimators}, the results would have likely improved 
even more if one were to add even a slight non-linearity.

In \cite{Alpaydin}, it was found that MIL specific objectives perform 
better than SI methods in cases with intra-bag dependency in the data. 
Here, it is important to distinguish between two situations. First, 
in part of the experiments in \cite{Alpaydin}, the label is not assigned 
to the bag by the rule which we discuss here: the bag is positive if and 
only if it contains at least one positive instance. These experiments 
simply investigate a different scenario. Second, in the 
experiments where the bag label \textit{is} assigned as above, only linear 
classifiers are evaluated. Again, one of the insights of this paper is 
that once we allow non-linearity, the results improve significantly

In \cite{bunescu2007multiple} it is argued that 
in some scenarios involving \textit{sparse} bags, SI methods may not 
perform well and alternative methods are proposed. Note that the example 
of images with caption labels does qualify as sparse bag situation. For 
instance, if ``frisbee'' is the label, an image may contain hundreds of 
patches, but only few of them would contain the frisbee. Nevertheless, 
in this paper we show that at least in the unbalanced data situation, 
bag sparsity is not an issue. All our experiments are with sparse bags, 
and in Section \ref{sec:si_analyis}, the parameter $l$ which controls 
sparsity, may be either small or large.

One possible approach to the MIL problem is to consider the SI label 
assignment as a \textit{noisy label} problem. The idea is that the 
assignment of a positive label to a negative object in a positive bag may 
be considered as label noise. One may then apply noisy label learning methods 
to recover the original labels. 
A variation of this approach was analyzed 
in terms of sample complexity in \cite{BlumMIL}, although that result does 
not lay itself to a practical algorithm. Another possibility is to use the 
noisy labels approach of \cite{Natarajan}. In \cite{Natarajan}, given the 
noisy label data, an new cost is constructed, such that the minimization 
of the new cost with respect to the noisy labels yields a classifier that 
is optimal with respect to the original labels. Unfortunately, both the 
arguments and the actual methods in both \cite{BlumMIL} and 
\cite{Natarajan} rely heavily on intra-bag independence. We refer 
to Section \ref{sec:feature_dep_bags} for additional details. In 
Section \ref{sec:niosy_label_estimators}, we show how the method 
based on the cost from \cite{Natarajan} fails, while the SI method does 
not, when the independence assumptions are violated.

\newcommand{\mppmt}{\mathcal{P'_{-}}}
\newcommand{\mnnt}{\mathcal{N'}}

\section{SI Analysis}
\label{sec:si_analyis}
In this section we present theoretical analysis of the SI method. Definitions are given in the following section. In Section \ref{sec:feature_dep_bags} we discuss the mixing intra-bag dependence. The theorems and their interpretations are presented in Section \ref{sec:results_results}.
\subsection{Definitions}
\label{sec:si_definitions}
In this section we analyze the SI algorithm for the MIL problem. 
The loss function for multiple classes will be obtained by summing the 
losses of each class individually, and therefore we discuss here the case
of single class with a binary label. 

The MIL dataset $\mathcal{S} = \Set{b}$ is given as a set of 
\textit{bags},  where each bag contains $M$ objects $x_j$, 
$b=\Set{x_1,\ldots,x_M}$, and for every bag there is a $0/1$-valued label
$y_b$. We assume that the labels belong to objects -- each object $x_i$ in 
$b$ has a label $y_{x_i}$, and we make the classical MIL assumption where $y_b = 1$ if and only if $y_{x_i} = 1$ for 
some $x_i$ in $b$. Our objective is to learn an instance-level classifier mapping from a single object $x_j$ to $0/1$, by employing the bag-level training data $\mathcal{S}$.

Denote by $\mathcal{P} = \Set{ b \in \mathcal{S} \setsep y_b = 1}$ 
the set of positive bags, and by 
$\mathcal{N} = \mathcal{S}\setminus \mathcal{P}$ the set of negative bags. 
We assume that each positive bag contains $l$ positive samples, where $l$ 
may be small compared to the size of the bag $M$. 

The balance of the dataset will be denoted by $B$,
\begin{equation}
B = \Abs{\mathcal{N}} / \Abs{\mathcal{P}}.
\end{equation}
The balance is the ratio negative to positive examples. For instance, 
in the COCO dataset, the label ``car'' has $B \sim 30$, while the label 
``fish'' has $B \sim 300$. Note that here we refer to the image level labels 
extracted from the captions, not the hard labels of the dataset, although 
balance numbers there are in general similar to those of similar labels in 
captions.

The \textit{unpacked} dataset $\mathcal{S'}$ is the collection of all 
objects from all the bags in $\mathcal{S}$. Denote by 
$\mathcal{P'_{+}}$ the collection of all positive objects in 
$\mathcal{S'}$, 
\begin{equation}
\mathcal{P'_{+}} = \Set{ x \in \mathcal{S'} \setsep y_x = 1},
\end{equation}
and similarly set 
\begin{eqnarray}
\mathcal{P'_{-}} &=& \Set{ x \in \mathcal{S'} \setsep x \in b \mbox{ such that } y_b = 1 \mbox{ and } y_x = 0}, \nonumber \\
\mathcal{N'} &=& \Set{ x \in \mathcal{S'} \setsep x \in b \mbox{ such that } y_b = 0}
\end{eqnarray}
In words, $\mathcal{P'_{+}}$ is the collection of positive objects from 
positive bags, $\mathcal{P'_{-}}$ are negative objects from positive bags, 
and $\mathcal{N'}$ are all negative objects from all negative bags. 

Denote $P = \Abs{\mathcal{P}}$. Then we have 
\begin{equation}
\Abs{\mathcal{P'_{+}}} = l P, \spaceo \Abs{\mathcal{P'_{-}}} = (M-l)P \mbox{ and }
\Abs{\mathcal{N'}} = MBP. 
\end{equation}

The \textit{ground truth label assignment} assigns label $1$ to objects in 
$\mathcal{P'_{+}}$ and $0$ to objects in $\mathcal{P'_{-}}$ and 
$\mathcal{N'}$. The assignment that is available to us is the 
\textit{SI assignment}, which gives label $1$ to objects in 
$\mathcal{P'_{+}}$ and $\mathcal{P'_{-}}$ and $0$ to objects in 
$\mathcal{N'}$.

\subsection{Feature Dependence in Bags}
\label{sec:feature_dep_bags}
As has been previously noted in the literature, the statistical 
distribution of features inside positive and negative bags can have 
a significant impact on performance of MIL algorithms. 
Empirical observations on datasets with different kinds of distributions 
may be found in \cite{Ray}. See also \cite{SabatoMIL} for connections 
of the MIL problem to NP-hardness in cases where no restrictions on 
distributions are imposed.

Here we first discuss two extreme cases, that of complete dependence and 
of independence. Then we discuss realistic cases in between, and the 
relation of the dependence to the data balance and the SI objective. 
Specifically, in what follows we will be interested in the relation 
between the distribution of objects in $\mppmt$ and $\mnnt$, the negative 
features in positive and negative bags.

To describe an example of complete dependence, consider a hypothetical 
COCO type dataset, with labels ``car'' and ``tree'' given at an image 
level, where each image is a bag of patches.  We are interested in an 
instance level classifier for ``car''. However, suppose that the dataset 
is such that ``car'' and ``tree'' either appear both in an image, or 
both of them do not appear. In that situation, without additional 
assumptions, it is clear that \textit{any} MIL classifier, with any 
objective, will have to classify any instance of ``tree'' as ``car''. 
This is simply since ``tree'' and ``car'' are indistinguishable from the 
label information. 

On the other hand, one may consider a situation where features 
in $\mppmt$ and $\mnnt$ are generated by the same distribution. We 
formulate this as the mixing assumption, for future reference. 

\begin{assume}[Mixing]
\label{assm:mixing}
Objects in $\mathcal{P'_{-}}$ are generated from the same distribution as 
those in $\mathcal{N'}$. 
\end{assume}

To understand this assumption, consider the label ``car'' in a more 
realistic dataset. The patches with cars will belong to $\mathcal{P'_{+}}$. Patches 
with, ``trees'', however, will belong to $\mathcal{P'_{-}}$ or 
to $\mathcal{N'}$ depending on whether they were extracted from image 
containing a car or not. The mixing assumption essentially states that 
the probability of observing tree in an image is independent of whether 
there is a car in the image, and also that impossible to distinguish 
between the type of trees that appear in car images and in images without 
cars. The types of patches one expects in an image are independent of 
whether a car is present in the image or not.

When the mixing assumption holds, it is generally known that an SI 
assignment translates the MIL problem into a \textit{noisy label} problem. 
One thinks of label $1$ on objects from $\mppmt$ as noise. Then, 
classification with noise methods, such as \cite{Natarajan},  may be 
applied. See \cite{BlumMIL} for a variation of this approach (under 
some additional assumptions on bag composition).  As we discuss further in 
Section \ref{sec:niosy_label_estimators}, noisy label estimators depend 
heavily on the mixing assumption. On the other hand, in this paper 
we show that if the data is unbalanced, then the straightforward 
supervised learning classifier from the SI assignment is extremely robust 
against violations in the mixing assumption, which is indeed violated in 
real datasets.

Indeed, consider finally the real COCO dataset. While the concepts of 
``car'' and ``tree'' may be independent, it is clear and easy to verify 
 that the appearance of ``car'' in the image is strongly (but 
not completely) positively correlated with the concept ``traffic light'' 
and strongly negatively correlated with concept ``bear''.

\subsection{Results}
\label{sec:results_results}
We assume that we work with classifiers that take values in the interval 
$[0,1]$, for instance classifiers of the form $f(x) = \sigma(g(x))$, 
where $g(x)$ is a logit of a neural network. 

\begin{thm}
\label{thm:main_thm}
Assume there is a classifier $f(x)$ which fits the ground truth assignment 
perfectly, $f(x) = 1$ for $x\in \mathcal{P'_{+}}$, and $f(x) = 0$ for 
$x\in \mathcal{P'_{-}}$ and $x\in \mathcal{N'}$. 
If the mixing assumption \ref{assm:mixing} holds, then the SI loss 
objective 
\begin{flalign}
\label{eq:si_objective}
&L(g) = \\
&-\sum_{x \in \mathcal{P'_{+}}} \log g(x) - 
\sum_{x \in \mathcal{P'_{-}}} \log g(x) -
\sum_{x \in \mathcal{N'}} \log \Brack{1-g(x)}  \nonumber,
\end{flalign}
is minimized by $f'(x)$ such that 
\begin{equation}
\label{eq:si_mixing_solution}
f'(x) = 
\begin{cases*}
1 & if $x \in \mathcal{P'_{+}}$ \\
\frac{\Abs{\mathcal{P'_{-}}}}{\Abs{\mathcal{P'_{-}}} + \Abs{\mathcal{N'}}} & if 
$x \in \mathcal{P'_{-}} \cup \mathcal{N'} $  
\end{cases*}
\end{equation}
\end{thm}
The loss (\ref{eq:si_objective}) corresponds to supervised 
learning with the SI label assignment. With our definitions, we have 
\begin{equation}
\label{eq:thereshold_value}
\frac{\Abs{\mathcal{P'_{-}}}}{\Abs{\mathcal{P'_{-}}} + \Abs{\mathcal{N'}}} = 
\frac{(M-l)P}{(M-l)P + MBP } \sim \frac{1}{B+1}.
\end{equation}
Therefore, by learning the SI objective, and thresholding the result 
at a value higher than $\frac{1}{B+1}$, we obtain a perfect classification 
with respect to the \textit{ground truth}. In particular $f'$ has the same 
precision-recall curve as $f$. Thus, if the rest of the assumptions hold, 
and the family of classifiers is rich enough to contain classifiers of the 
form $f'$ we can obtain instance level classification from bag level 
labels and an SI assignment. Note that $f'$ is simply a linear 
modification of $f$, 
\begin{equation}
\label{eq:f_prime_linear_form}
f'(x) = f(x) + (1-f(x))\cdot \frac{\Abs{\mathcal{P'_{-}}}}{\mathcal{\Abs{P'_{-}}} + \Abs{\mathcal{N'}}} 
\end{equation}
and we assume $f$ is in the family. On the other hand, note also that 
if $f$ is, say, a logistic regression, then $f'$ is no longer exactly 
realizable by a logistic regression. 
See also Section \ref{sec:niosy_label_estimators} for an additional 
discussion and an example where the richness of the class plays a role.

We now prove Theorem \ref{thm:main_thm}. 
\begin{proof}
First, to minimize (\ref{eq:si_objective}), it is clear that one 
has to set $g(x) = 1$ for $x \in \mathcal{P'_{+}}$. 
Our objective is therefore to show that the two other terms of 
(\ref{eq:si_objective}) are minimized by a constant value 
$g(x) = \frac{\Abs{\mathcal{P'_{-}}}}{\Abs{\mathcal{P'_{-}}} + \Abs{\mathcal{N'}}}$. 

Let $x$ be sampled from $\mathcal{P'_{-}}$. Then $g(x)$ is a scalar random 
variable, with some distribution $G$. By the mixing assumption, $g(x)$  
will have the same distribution $G$ when $x$ is sampled from  
$\mathcal{N'}$. We can therefore rewrite the last two terms of 
(\ref{eq:si_objective}) as 
\begin{flalign}
&\Abs{\mathcal{P'_{-}}}\cdot  \frac{1}{\Abs{\mathcal{P'_{-}}}} \cdot \sum_{x \in \mathcal{P'_{-}}} \log g(x) +
\Abs{\mathcal{N'}}\cdot  \frac{1}{\Abs{\mathcal{N'}}}
\sum_{x \in \mathcal{N'}} \log \Brack{1-g(x)}  \nonumber \\
&= \Abs{\mathcal{P'_{-}}} \Expsubidx{g \sim G}{ \log g} +
\Abs{\mathcal{N'}}\cdot  \Expsubidx{g \sim G}{ \log \Brack{1-g} } \nonumber \\
&= \Expsubidx{g \sim G}{\Brack{ \Abs{\mathcal{P'_{-}}}  \log g +
\Abs{\mathcal{N'}}\cdot  \log \Brack{1-g} }},
\label{eq:exp_G_si_loss}
\end{flalign}
where we have also removed the minus sign, and we seek to maximize 
(\ref{eq:exp_G_si_loss}) over all possible distributions $G$. We have used 
the identity of the distribution of $g(x)$ on $\mathcal{P'_{-}}$ and 
$\mathcal{N'}$ in the passage from the first to the second line. In this passage we have also assumed that 
sample averages may be replaced by respective expectations, that is, that we work in the large sample limit. A more detailed discussion of this assumption may be found in Section \ref{sec:conclusion}.

Next, one readily verifies that the function
\begin{equation}
r(g) =  \Abs{\mathcal{P'_{-}}}  \log g +
\Abs{\mathcal{N'}}\cdot  \log \Brack{1-g} 
\end{equation}
with $g \in (0,1)$ is maximized at $g = \frac{\Abs{\mathcal{P'_{-}}}}{\Abs{\mathcal{P'_{-}}} + \Abs{\mathcal{N'}}}$. This can be seen either directly by taking the derivative, or as a consequence of the non-negativity of the Kullback-Leibler divergence between the distributions 
on two points given by $(\frac{\Abs{\mathcal{P'_{-}}}}{\Abs{\mathcal{P'_{-}}} + \Abs{\mathcal{N'}}}, \frac{\Abs{\mathcal{N'}}}{\Abs{\mathcal{P'_{-}}} + \Abs{\mathcal{N'}}})$ and $(g,1-g)$. 
It therefore follows that (\ref{eq:exp_G_si_loss}) is maximized when 
$G$ is an atomic distribution taking value 
$\frac{\Abs{\mathcal{P'_{-}}}}{\Abs{\mathcal{P'_{-}}} + \Abs{\mathcal{N'}}}$ with probability 1. 
\end{proof}

We now analyze the case where the mixing assumption does not hold. 
\newcommand{\mupp}{\mu_{\mathcal{P'_{-}}}}
\newcommand{\munn}{\mu_{\mathcal{N'}}}
\begin{thm}
\label{thm:main_thm_gen}
Denote by $\mupp(x)$ and 
$\munn(x)$ the distributions form which 
objects in $\mathcal{P'_{-}}$ and $\mathcal{N'}$ are generated, 
respectively. Assume there is a classifier $f(x)$ which fits the ground 
truth assignment perfectly, $f(x) = 1$ for $x\in \mathcal{P'_{+}}$, 
and $f(x) = 0$ for $x\in \mathcal{P'_{-}}$ and $x\in \mathcal{N'}$.

Then the SI loss objective 
\begin{flalign}
\label{eq:si_objective_g}
&L(g) = \\
&-\sum_{x \in \mathcal{P'_{+}}} \log g(x) - 
\sum_{x \in \mathcal{P'_{-}}} \log g(x) -
\sum_{x \in \mathcal{N'}} \log \Brack{1-g(x)}  \nonumber,
\end{flalign}
is minimized by $f'(x)$ such that 
\begin{equation}
\label{eq:si_non_mixing_solution}
f'(x) = 
\begin{cases*}
1 & if $x \in \mathcal{P'_{+}}$ \\
\frac{\Abs{\mathcal{P'_{-}}}\mupp(x)}{\Abs{\mathcal{P'_{-}}}\mupp(x) + \Abs{\mathcal{N'}}\munn(x)} & if 
$x \in \mathcal{P'_{-}} \cup \mathcal{N'} $  
\end{cases*}
\end{equation}
\end{thm}
\begin{proof}
As in the proof of Theorem \ref{thm:main_thm}, the existence of a 
separating classifier $f$ implies that $\mathcal{P'_{+}}$ and 
$\mathcal{P'_{-}} \cup \mathcal{N'}$ are disjoint, and similarly we set 
$f'(x) = 1$ for $x \in \mathcal{P'_{+}}$. We now consider the last two 
terms of the cost (\ref{eq:si_objective_g}), and 
$x \in \mathcal{P'_{-}} \cup \mathcal{N'}$. 
Rewrite the terms in (\ref{eq:si_objective_g}) as 
\begin{equation}
\label{eq:E1_def}
E_1(g) := \sum_{x \in \mathcal{P'_{-}}} \log g(x)  = 
\Abs{\mathcal{P'_{-}}}\cdot\Expsubidx{ x \sim \mupp }{ \log g(x)}
\end{equation}
and 
\begin{flalign}
&E_2(g) := \nonumber \\
&\sum_{x \in \mathcal{N'}} \log \Brack{1-g(x)} = 
\Abs{\mathcal{N'}}\cdot \Expsubidx{ x \sim \munn}{\Brack{1- \log g(x)}}.
\label{eq:E2_def}
\end{flalign}
Define the mixture $\hat{\mu}(x)$ by
\begin{equation}
\hat{\mu}(x) = 
\Brack{\Abs{\mathcal{P'_{-}}} \mupp(x) + \Abs{\mathcal{N'}} \munn(x) } \big/ 
\Brack{\Abs{\mathcal{P'_{-}}} + \Abs{\mathcal{N'}}}. 
\end{equation}
Since both $\mupp$ and $\munn$ are absolutely continuous with respect to 
$\hat{\mu}$, we can change the measure to write
\begin{equation}
E_1(g) = \Expsubidx{x \sim \hat{\mu}}{
 \frac{\Abs{\mathcal{P'_{-}}}\mupp(x)}{\hat{\mu}(x)} \log g(x) 
}
\end{equation}
and 
\begin{equation}
E_2(g) = \Expsubidx{x \sim \hat{\mu}}{
 \frac{\Abs{\mathcal{N'}}\mupp(x)}{\hat{\mu}(x)} \Brack{1-\log g(x) }
}. 
\end{equation}
Thus, 
\begin{flalign}
&E_1(g) + E_2(g) = \\
&
\Expsubidx{x \sim \hat{\mu}}{\Brack{
 \frac{\Abs{\mathcal{P'_{-}}}\mupp(x)}{\hat{\mu}(x)} \log g(x) 
+
 \frac{\Abs{\mathcal{N'}}\mupp(x)}{\hat{\mu}(x)} \Brack{1-\log g(x) }
 } \nonumber
}.
\end{flalign}
It remains to observe that similarly to the argument in Theorem 
\ref{thm:main_thm}, for each fixed $x$, the 
expression
\begin{equation}
 \frac{\Abs{\mathcal{P'_{-}}}\mupp(x)}{\hat{\mu}(x)} \log g(x) 
+
 \frac{\Abs{\mathcal{N'}}\mupp(x)}{\hat{\mu}(x)} \Brack{1-\log g(x) }
\end{equation}
is maximized over $g$ at $x$ iff 
\begin{equation}
g(x) = \frac{\Abs{\mathcal{P'_{-}}}\mupp(x)}{\hat{\mu}(x)},
\end{equation}
which concludes the proof. 
\end{proof}
Note that Theorem \ref{thm:main_thm_gen} is a proper generalization of Theorem \ref{thm:main_thm}. However, we chose to presented them separately due to illustrative purposes.

As discussed earlier, Theorem \ref{thm:main_thm_gen} reveals the real 
power of the SI method in the unbalanced case. 
Consider the expression for $f'(x)$ in (\ref{eq:si_non_mixing_solution}), 
for $x \in \mppmt \cup \mnnt$:
\begin{flalign}
&f'(x) = \\
&\frac{\Abs{\mppmt}\mupp(x)}{\Abs{\mppmt}\mupp(x) + \Abs{\mnnt}\munn(x)} =  \\
&\frac{1}{1 + \frac{\Abs{\mnnt}\munn(x)}{ \Abs{\mppmt}\mupp(x)} }. 
\end{flalign}
In terms of Theorem \ref{thm:main_thm_gen}, the mixing assumption 
\ref{assm:mixing} is equivalent to the assertion $\mupp(x) = \munn(x)$ for 
all $x$. In this case, 
the term 
\begin{equation}
\frac{1}{1 + \frac{\Abs{\mnnt}\munn(x)}{ \Abs{\mppmt}\mupp(x)} }
\end{equation}
reduces to 
\begin{equation}
\frac{1}{1 + \frac{\Abs{\mnnt}}{ \Abs{\mppmt}} } \sim 
\frac{1}{1 + B}. 
\end{equation}
As discussed previously, this allows us to place a decision threshold 
above $\frac{1}{1 + B}$ and make a perfect classification. Next, 
if $\mupp(x) < \munn(x)$, then 
\begin{equation}
\frac{1}{1 + \frac{\Abs{\mnnt}\munn(x)}{ \Abs{\mppmt}\mupp(x)} } 
< \frac{1}{1 + \frac{\Abs{\mnnt}}{ \Abs{\mppmt}} }
\end{equation}
and therefore the same decision threshold will still result in 
correct classification. These therefore are the easier cases. Consider now 
what happens when $\mupp(x) > \munn(x)$. The extreme case discussed above 
of ``tree'' appearing in the image if and only if ``car'' appears 
 would correspond 
to $\mupp(x) > 0$ and $\munn(x) = 0$ for features $x$ corresponding to 
``tree''. Therefore one asks how much larger $\mupp(x)$ can be compared to $\munn(x)$. Suppose that we wish to place the decision threshold at 
$\frac{1}{2}$. Then $f'(x)\leq \half$ iff 
\begin{equation}
B \munn(x) \sim  \frac{\Abs{\mnnt}}{\Abs{\mppmt}} \munn(x) \geq \mupp(x).
\end{equation}
Therefore, the frequency of $x$ in $\mppmt$ can be up to $B$ times larger 
than that in $\mnnt$ and still obtain the right classification. In other 
words, the lack of balance in the data provides a large margin in which 
the mixing assumption may not hold. The larger the imbalance $B$ is, the 
larger dependence in features the SI method can tolerate.
Therefore \textbf{in typically unbalanced MIL dataset, SI is a robust 
classification method}.  

To conclude this section let us make a few notes regarding  the 
separability assumption -- the assumption in Theorems \ref{thm:main_thm} 
and \ref{thm:main_thm_gen} that there is a classifier $f$ which separates 
$\mathcal{P'_{+}}$ and $\mppmt \cup \mnnt$ perfectly. One could consider a more general case 
where the optimal, in terms of cross-entropy cost, classifier of the 
unpacked dataset has a precision-recall curve that is not identically one
(and hence has an average precision score smaller than $1$). 
This could happen for instance if the features are not strong enough to 
completely separate the classes. If the mixing assumption holds, arguments 
similar to those of Theorem \ref{thm:main_thm} imply that the  $f'$
learned from the SI assignment would still have the form 
(\ref{eq:f_prime_linear_form}), and, since this form is monotone in $f$, 
would have the same precision-recall curve as $f$.  When the mixing 
assumption does not hold, instead of considering the ratio of densities 
between the positive and negative classes, one would have to consider the 
ratios at all level sets of $f$. While this would significantly complicate 
the notation, conclusions similar to those of Theorem 
\ref{thm:main_thm_gen} would still hold.

\section{Experiments}
\label{sec:experiments}

\subsection{Non Linearity And Noisy Label Estimator}
\label{sec:niosy_label_estimators}
In this Section we evaluate the SI method on a data where the mixing assumption does not hold. We demonstrate the utility of adding a non-linearity. In addition, as described in Section \ref{sec:feature_dep_bags}, we evaluate the noisy label cost from \cite{Natarajan}, referred to as UC, and show that it does not perform well when the mixing assumption does not hold. 

\newcommand{\spp}{\mathcal{P'}_{+}}
We work with the unpacked dataset (as defined in Section \ref{sec:si_definitions}) corresponding to values 
$M=100$, $\ell=1$, $B=20$ and $P=100$. The sets for $\spp$ and $\mnnt$ are shown in Figure \ref{fig:clean_data}. 
The set $\spp$ is located on the line $(x,-0.5)$ with $x$ uniformly distributed in $[-2,2]$,  $x \sim \mathcal{U}  [-2,2]$. The set $\mnnt$ is split evenly between two intervals. The first half is located on a line $(-1,y)$ with $y \sim \mathcal{U}  [0,5]$ and the second half is located on a line $(1,y)$ with $y \sim \mathcal{U}  [0,5]$. In order to break the mixing assumption, $80\%$ of the points from $\mppmt$ are distributed on the line $(1,y)$ and $20\%$ on the line $(-1,y)$ with $y \sim \mathcal{U}  [0,5]$ in both cases.  Note that the mixing assumption would correspond to a $50\% - 50\%$ split. Clearly, this dataset is linearly separable, e.g. by the line $y=-0.25$. The noisy data is illustrated in Figure \ref{fig:noisy_data}. Note that for clarity only a small fraction of the points appear on the plots.

For both SI and UC we trained two models on the data with the SI assignment labels (Figure \ref{fig:noisy_data}). The first model is a one layer neural network, i.e a linear model. The second model is a two layers neural network, with two neurons in the hidden layer and sigmoid activations.
We trained each model for $100000$ epochs \footnote{We also tried to run more epochs but it did not change the conclusions.} with the ADAM optimizer where we set batch size equals to dataset size. We trained with a constant learning rate in $\{10^{-4}, 10^{-5}, 10^{-6}\}$ and choose the classifier achieving the lowest training loss.

The average precision scores\footnote{Computed with \texttt{average\_precision\_score} function from sklearn.metrics.}
of the resulting classifiers with respect to the true labels (Figure \ref{fig:noisy_data}), are shown in Table \ref{tab:AP}.

In Figure \ref{fig:heatmaps} the prediction score (the output sigmoid of the model) is shown as a heat-map for each case. We first note that although the UC cost is theoretically guaranteed to find the correct classifier when mixing holds, here it fails in both architectures. For the SI cost, observe that the linear classifier approximates only 
poorly the optimal SI classifier $f'$, \ref{eq:si_non_mixing_solution}. However, the two layer model approximates 
$f'$ much better (Figure \ref{fig:heatmaps}, bottom left) and thresholding it at an appropriate level separates 
ground truth positives from negatives perfectly, therefore yielding the AP score $1$.

\begin{table}
\begin{center}
\begin{tabular}{|l|c|c|}
\hline
Method & One Layer & Two Layers  \\
\hline\hline
SI & 0.21 & 1 \\
Unbiased Estimator & 0.30 & 0.23 \\
\hline
\end{tabular}
\end{center}
\caption{Average precision scores of the optimizer of SI and UC costs, for linear and two layer models.}
\label{tab:AP}
\end{table}

\begin{figure*}
\centering
\subfigure[Clean data]{\label{fig:clean_data}\includegraphics[width=.49\linewidth, height = 5cm]{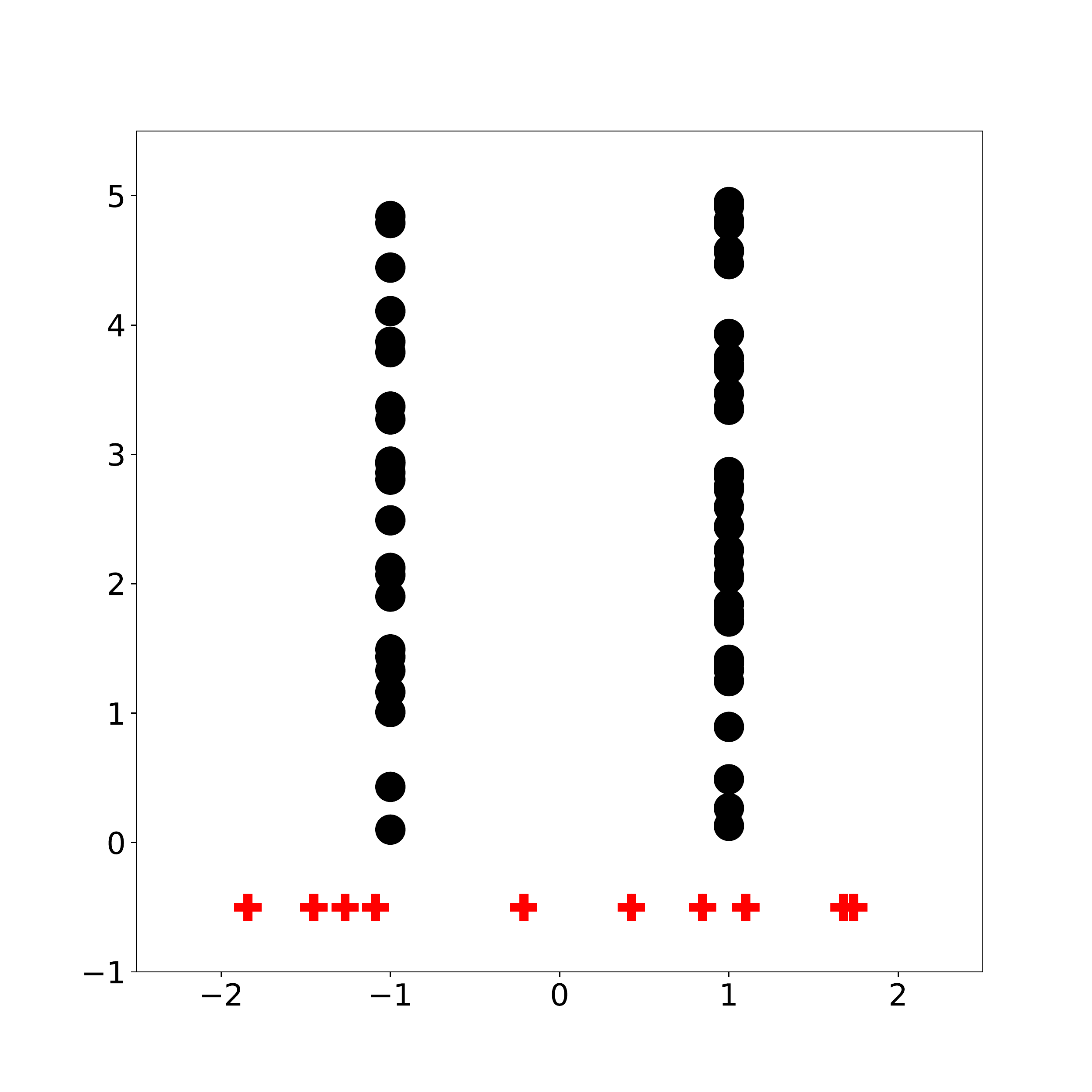}}
\subfigure[Noisy data]{\label{fig:noisy_data}\includegraphics[width=.49\linewidth, height = 5cm]{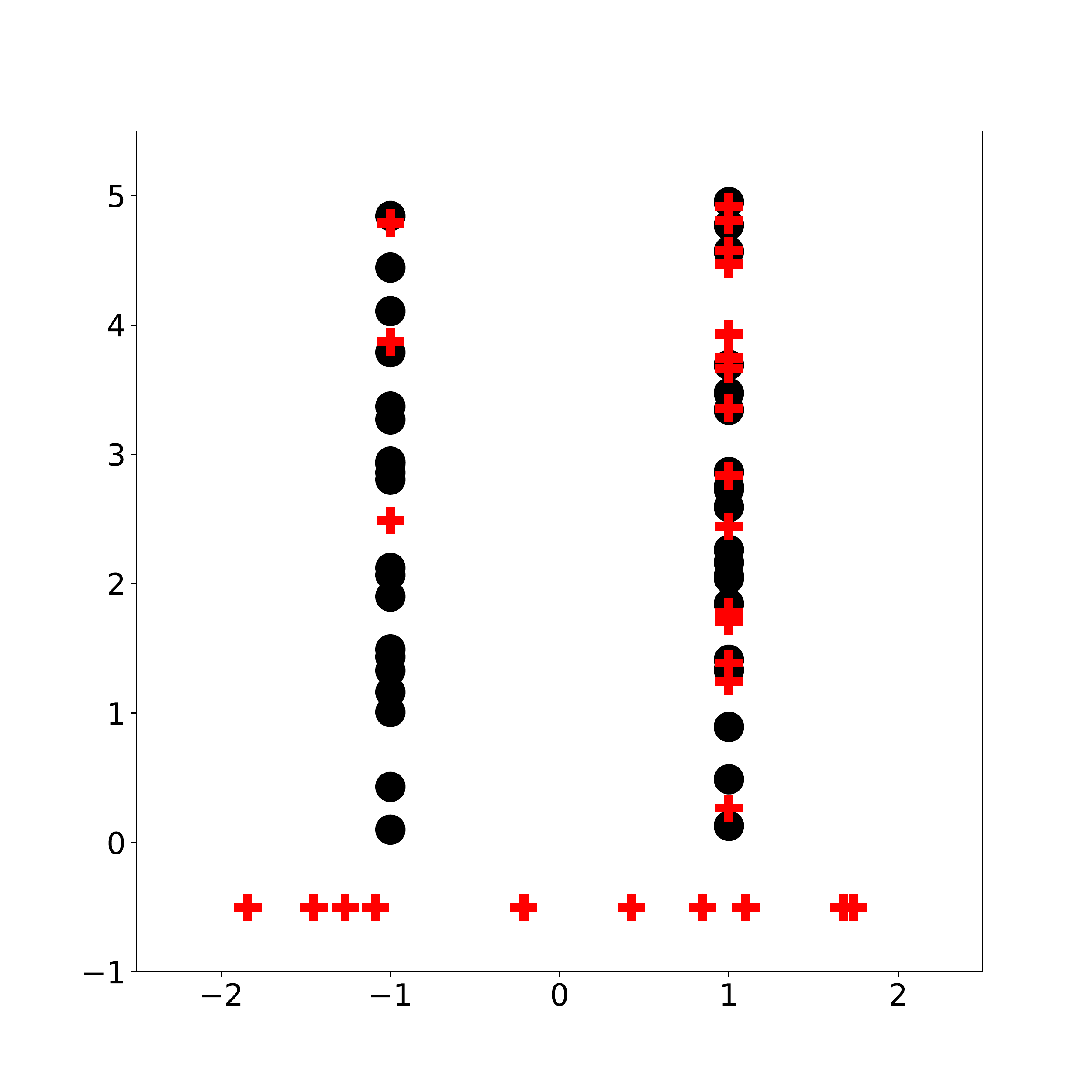}}
\caption{Illustration of the data used in noisy labels experiment. (a) - clean data, (b) - noisy data used for training. Only a small fraction of the data is illustrated (best viewed in color).}
\end{figure*}

\begin{figure*}[t]
\centering
\subfigure[One layer, SI]{\label{fig:1L_SI}\includegraphics[width=.4\linewidth,height = 4cm]{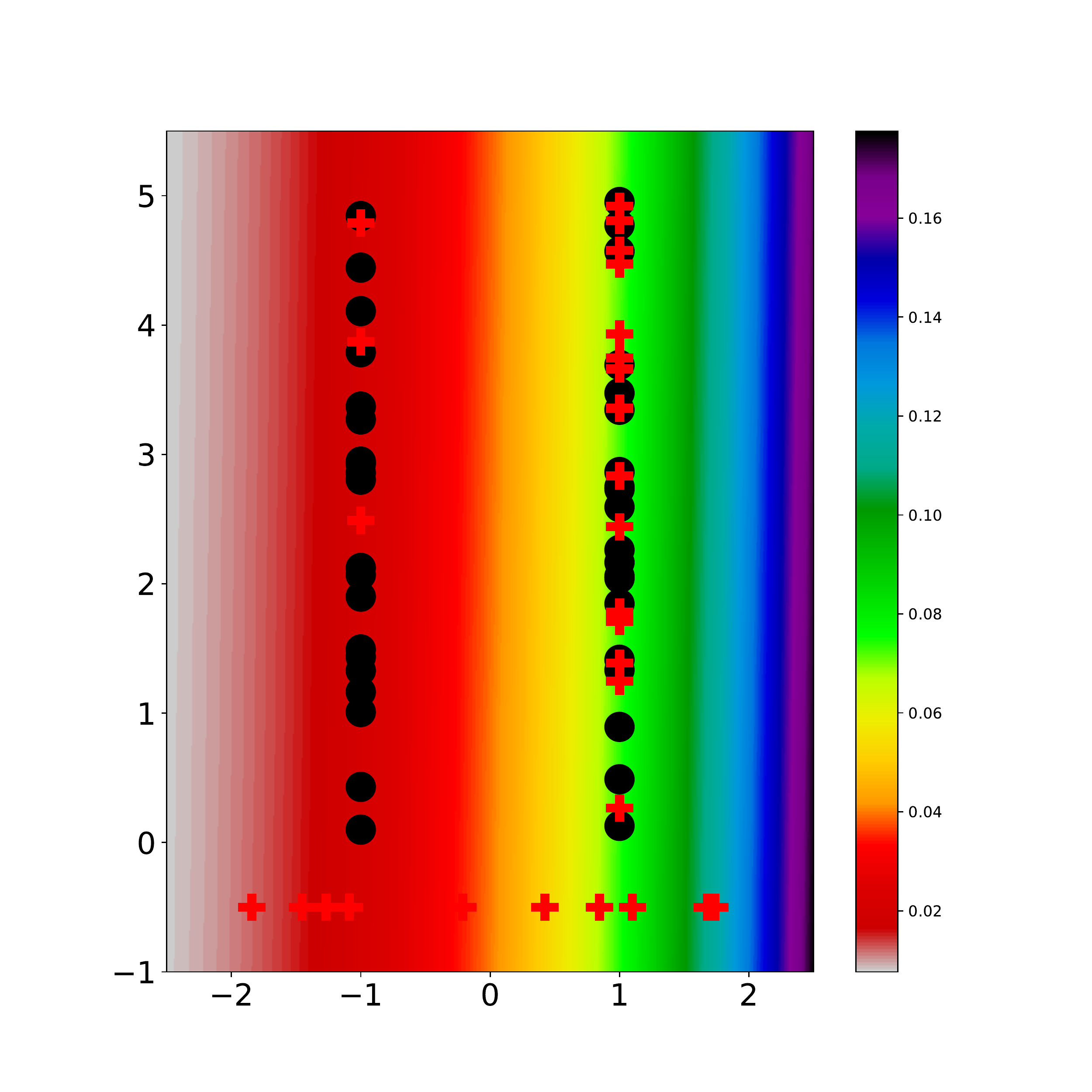}}
\subfigure[One layer, UC]{\label{fig:1L_unbiased}\includegraphics[width=.4\linewidth,height = 4cm]{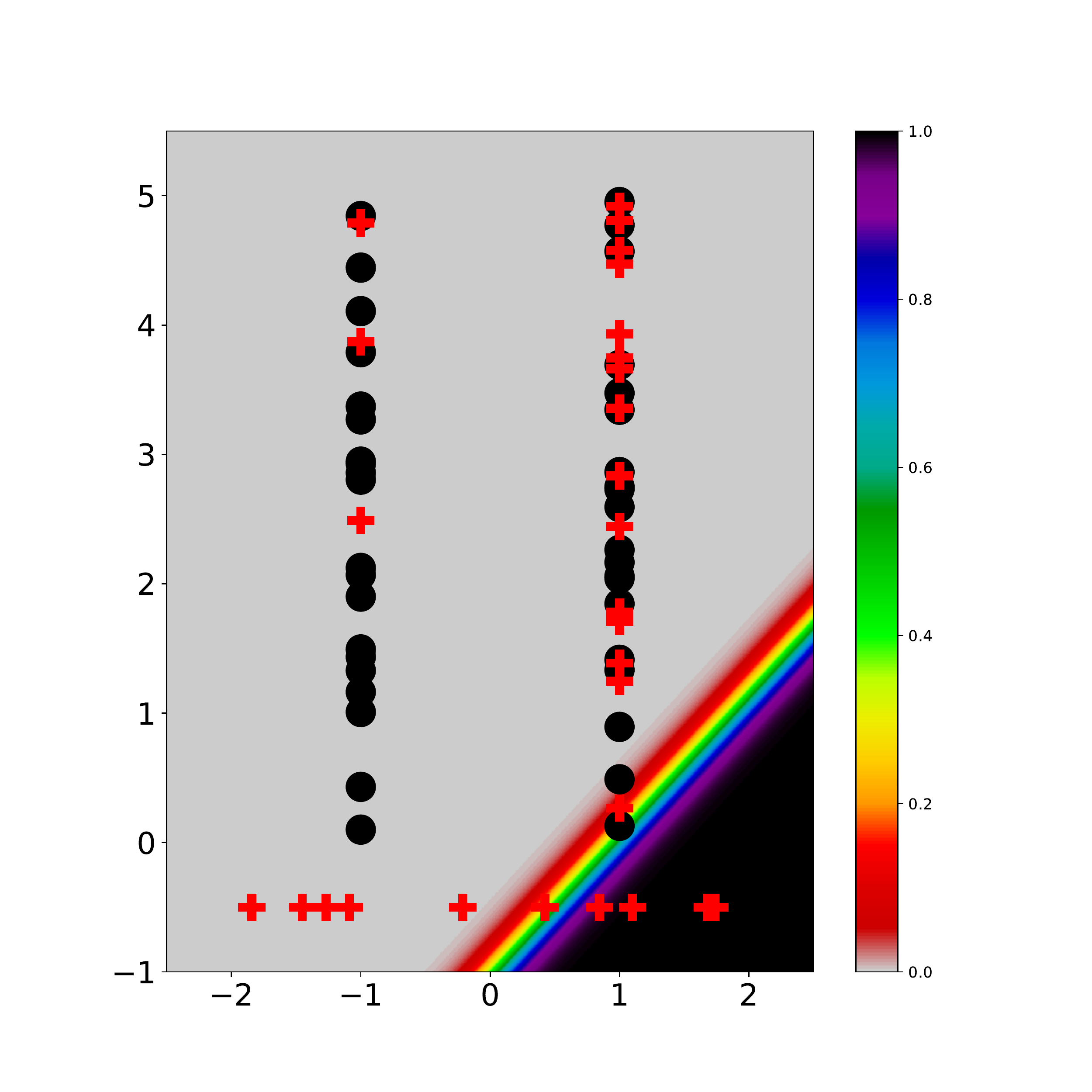}}
\subfigure[Two layers, SI]{\label{fig:2L_SI}\includegraphics[width=.49\linewidth,height = 4cm]{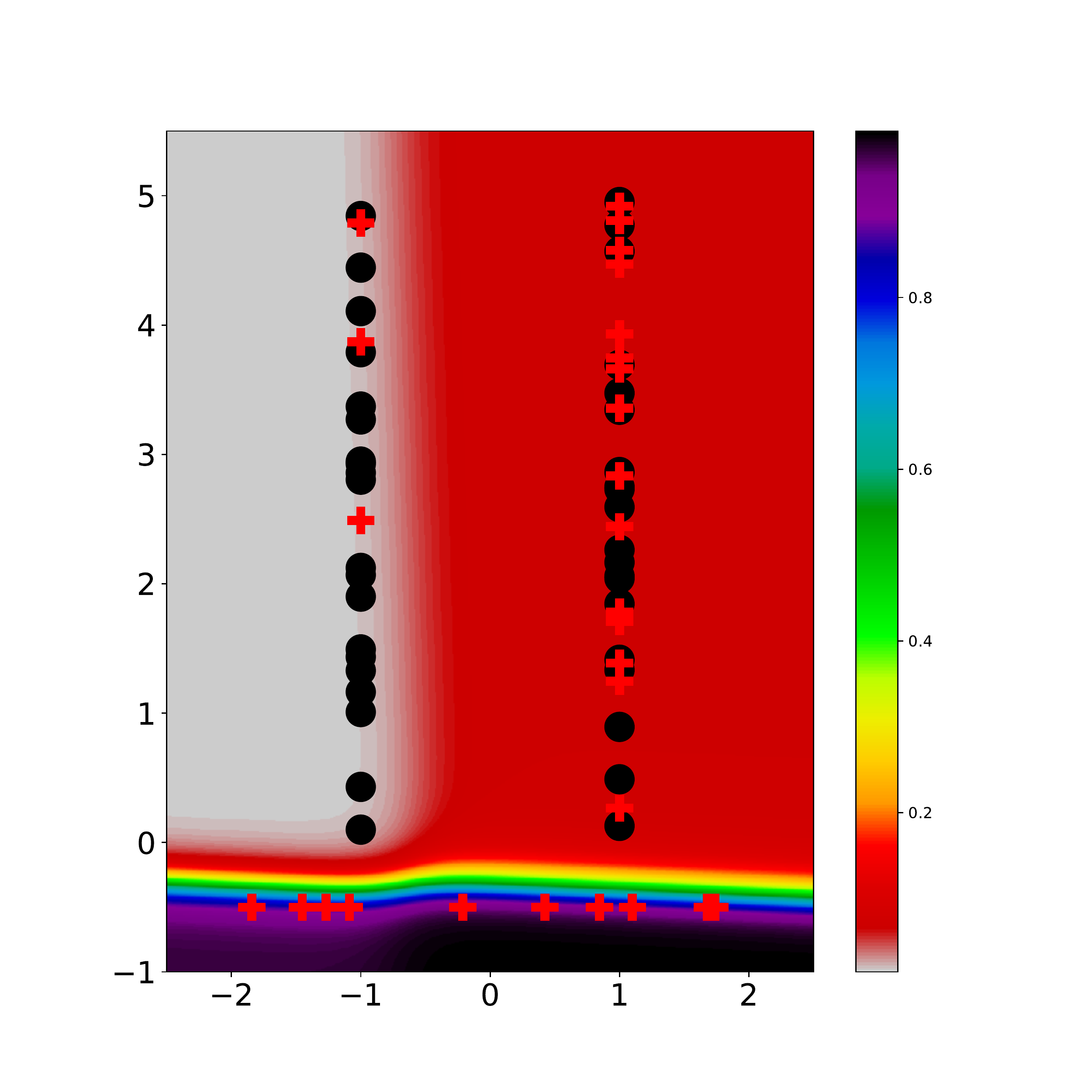}}
\subfigure[Two layers, UC]{\label{fig:2L_unbiased}\includegraphics[width=.49\linewidth,height = 4cm]{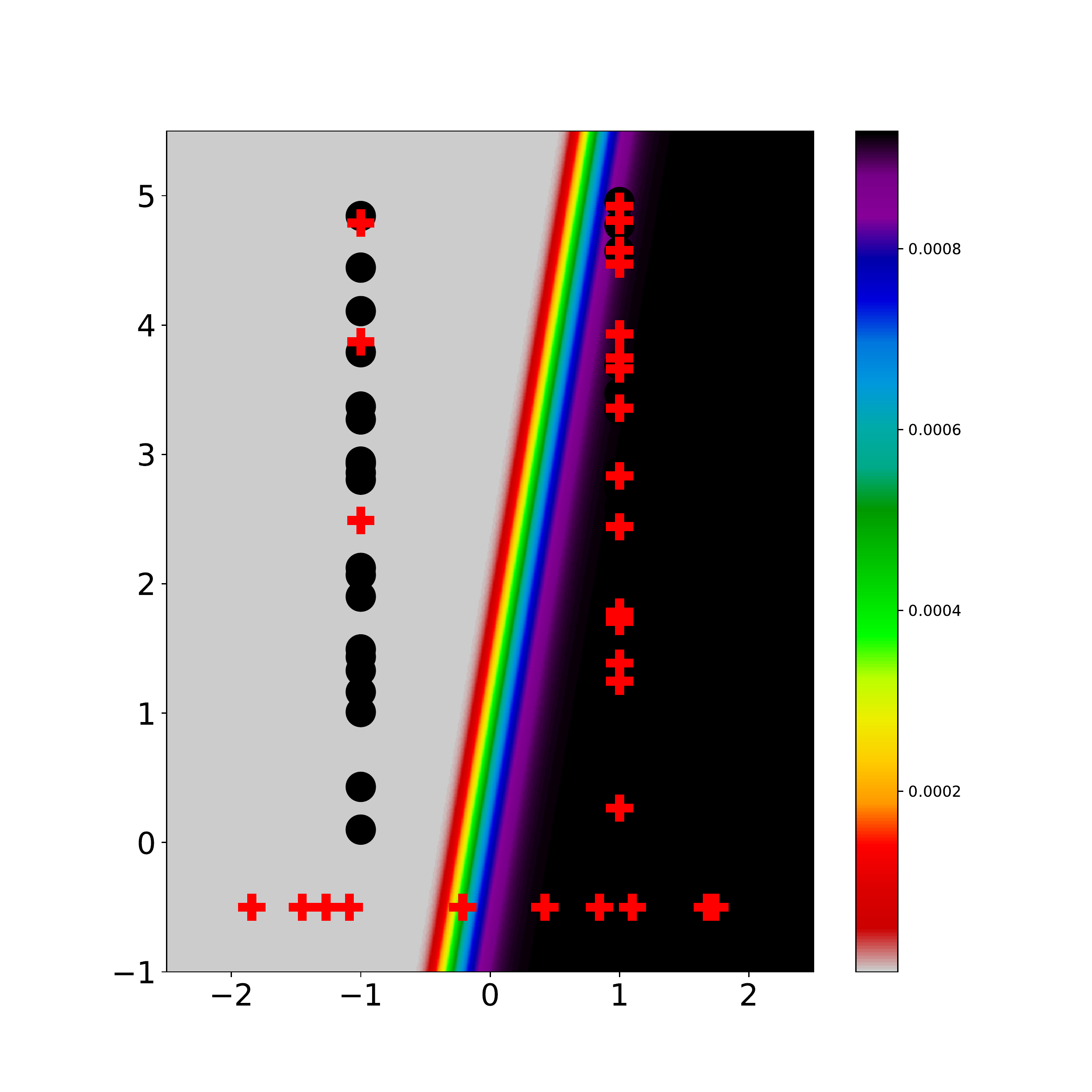}}
\caption{Heat-maps representing scores of the models trained with SI and UC costs (best viewed in color).
}
\label{fig:heatmaps}
\end{figure*}

\subsection{COCO}
As described in the Introduction, we consider the problem of object 
classification from captions data on the COCO 2014 dataset \cite{COCO_ds}. 
This problem can be naturally interpreted as an MIL problem. 

We adopt the experimental setting of \cite{Fang_2015_CVPR}, and 
compare the performance of the SI classifier to the performance 
of the MIL objective used in \cite{Fang_2015_CVPR}.

In this setting, each image is rescaled to a size of $576 \times 576$,
and divided into $12\times12$ patches of size $224$ with stride $32$. Each 
image is therefore a bag containing $144$ objects. Each image was fed into 
a VGG16 network, \cite{vgg16}. The output of the $fc7$ layer is then 
a $4096$-dimensional representation for each patch. Next, a convolutional 
layer with $(1,1)$ stride and $1000$ filters is used to represent 
classifiers from patch features, for a $1000$ labels. We refer to 
\cite{Fang_2015_CVPR} for full architectural details. 

The image labels are derived from captions. No preprocessing was done on 
captions, except a conversion to lower case. The vocabulary of labels 
consists of  $1000$ most frequent tokens appearing in captions. Note 
that about $50$ of these tokens are stopwords. However, to allow direct 
comparison to the code of \cite{Fang_2015_CVPR}, we chose to maintain the 
same vocabulary, and to measure the performance on all the labels, and 
separately on a selected subset of labels, as discussed below. 

For a fixed label $z$, let $f_z(x)$ be the sigmoid output of a classifier 
corresponding to $z$.  For patches $x_1, \ldots,x_{144}$ of an image $b$, 
the MIL objective used in \cite{Fang_2015_CVPR} corresponding to the 
image is 
\begin{equation}
\label{eq:fang_objective}
o_z(b) = 1 - \prod_{i=1}^{144} \Brack{1 - f_z(x_i)}
\end{equation}
and the total cost term corresponding to $b$ is obtained by summing the 
cost over all labels, 
\begin{equation}
\label{eq:total_cost_fang}
c(b) = \sum_z ce(o_z(b),y_z(b)),
\end{equation}
where $y_z(b)$ is the indicator of the label and $ce$ is the cross-entropy 
cost. 
The SI objective for the image $b$ is given by 
\begin{equation}
\label{eq:si_total_cost_coco}
c(b) = \sum_z \sum_{i=1}^{1000} ce(f_z(x_i),y_z(b)).
\end{equation}

We have evaluated the performance at the bag level. Specifically, 
for a label $z$, and image $b$, given the scores $f_z(x_i)$ we construct 
a bag level score $s_z(b)$ via 
\begin{equation}
s_z(b) = \max_i f_z(x_i). 
\end{equation}
Then we evaluate the Average Precision of the scores $s_z(b)$ against the 
labels $y_z(b)$ on the COCO \textit{eval} set. The mean Average Precision 
(mAP) is the mean over all labels $z$. In addition, as discussed above, 
since some labels are stopwords, and some labels appear very few times in the 
dataset, we also measure the mAP on a smaller subset of ``strong labels''. 
These are the labels such that their token appears as one of the object 
categories COCO, since these tend to be better represented 
in the dataset.  For instance ``car'' is a strong label, 
while ``water'' is not. The matching between caption labels and categories 
was done via text matching. Since some categories are described by two 
words (ex. ``traffic light''), they were not included. This process 
generated $63$ labels. It is important to note that object categories were 
only used to select the subset of labels. Training and evaluation of all models were performed solely using the images and caption data. 

To obtain the 
results for the MIL objective 
(\ref{eq:fang_objective}) we have used the code from 
\cite{Fang_2015_CVPR}, available online. These results were reproduced 
in our own code, implemented in Tensorflow. To obtain the results for the 
SI objective, we replaced the cost with the SI cost 
(\ref{eq:si_total_cost_coco}) in our implementation. The models were 
trained for $6$ epochs, at which point both of the models 
converged.

The results are given 
in Table \ref{tab:results_coco}. One can see that the results are close, 
although the MIL (\ref{eq:fang_objective}) results are slightly higher. 
We believe that the difference is due to the hyper-parameters rather 
than due to intrinsic properties of the costs involved. 
We have not attempted any hyper-parameter tuning due to the high 
computational cost of this operation. 
Instead, we have used the given heavily tuned hyper-parameters of the 
\cite{Fang_2015_CVPR} code (hardcoded bias term initializations, SGD 
learning rate type and decay, hardcoded varying training rates for 
different layers). These hyper-parameters were designed for the original 
objective, but are not necessarily optimal for SI. 

\begin{table}
\begin{center}
\begin{tabular}{|l|c|c|}
\hline
Method & All labels & Strong Subset  \\
\hline\hline
MIL(\ref{eq:fang_objective}) & 0.30 & 0.59 \\
SI & 0.26 & 0.56 \\
\hline
\end{tabular}
\end{center}
\caption{Comparison of mAP of MIL(\ref{eq:fang_objective}) 
and SI objectives on all labels, and on the strong labels subset.}
\label{tab:results_coco}
\end{table}

\section{Conclusions and Future Work}
\label{sec:conclusion}
We have shown that SI learning is an effective classification method for 
MIL data if the problem has the following characteristics: (a) The bag 
labels are derived from objects, in the sense that a bag is positive if 
and only if it contains a positive object. (b) The data is unbalanced -- 
there are more negative bags than positive. This allow the classification 
to be tolerant to a significant amount of dependence in the bags. (c) The 
class of classifiers is rich enough to contain not only the reference 
ground truth classifier, but also the classifiers $f'$ described by 
Theorems \ref{thm:main_thm} and \ref{thm:main_thm_gen}.

We now describe two possible directions for future work. From the 
theoretical perspective, our results are large-sample limit results. In 
particular we have assumed that we may replace sample averages by the 
respective expectations, as was done in (\ref{eq:exp_G_si_loss}) and 
(\ref{eq:E1_def}),(\ref{eq:E2_def}). While these computations allow us to 
understand the essential features of the problem, it is still an 
intriguing question of what can be said at the sample level. Classically, 
such questions may be answered within the framework of bounded complexity 
classifier classes, via notions such as the Rademacher complexity. 
However, these notions are well known not to be an adequate measure of 
complexity for neural networks, and therefore one must find a different 
approach.

From the practical perspective, the most appealing feature of SI method 
is the ability in principle to deal with arbitrarily large bags. As 
discussed earlier, typical MIL objectives compute a score, such as 
(\ref{eq:fang_objective}) which depends on \textit{all} objects in the 
bag. Therefore, one either has to be able to have all objects in memory at 
once, or to design a cumbersome architecture to compute such a score 
sequentiality. The SI approach on the other hand, does not have this 
problem. Note that large bags may appear naturally in applications. 
Consider for example the situation where a news article is treated as a 
bag, containing \textit{several} images. Even for a relatively modest 
number of images, keeping several copies of a modern visual CNN in memory 
is already prohibitive. We hope that the considerations in this paper 
shed light on the usefulness of the SI method and therefore open the door 
for such applications.

\bibliographystyle{ieee}
\bibliography{mil}

\end{document}